%% file: main.tex
\documentclass{article}
\input{math_commands}

\usepackage{microtype}
\usepackage{graphicx}
\usepackage{subcaption}
\usepackage{booktabs}
\usepackage{hyperref}
\usepackage{enumitem}
\usepackage[none]{hyphenat}
\usepackage[capitalise]{cleveref}
\usepackage{xcolor}
\crefname{equation}{}{}
\Crefname{equation}{}{}

\usepackage[accepted]{icml2023}

\icmltitlerunning{Estimation--Prediction Tradeoff in Causal Probabilistic Temporal Graphs}

\begin{document}

\twocolumn[
\icmltitle{
Estimation--Prediction Tradeoff in Causal Probabilistic Temporal Graphs}
\icmlsetsymbol{equal}{*}

\begin{icmlauthorlist}
\icmlauthor{Aniq Ur Rahman}{ox}
\end{icmlauthorlist}

\icmlaffiliation{ox}{Department of Engineering Science, University of Oxford, Oxford, OX1 3PJ, UK}

\icmlcorrespondingauthor{Aniq Ur Rahman}{aniq.rahman@eng.ox.ac.uk}

\icmlkeywords{binary classification, Fisher information, entropy, hypothesis testing, foundational limits}

\vskip 0.3in
]

\printAffiliationsAndNotice{}

\begin{abstract}
Temporal link prediction (TLP) is typically evaluated by predictive performance on unseen edges, but this criterion can conflate predictive accuracy with recovery of the underlying causal mechanism. In stochastic models, Fisher information governs the Cram\'{e}r--Rao (CR) bound on parameter estimation error: higher Fisher information permits more accurate parameter recovery. We show that, under comonotonicity conditions between Fisher information and entropy, binary logistic models exhibit an estimation--prediction tradeoff: regimes with higher Fisher information, and hence smaller CR bounds, also have higher irreducible predictive entropy. To study this tradeoff in TLP, we introduce a probabilistic causal generator for temporal graphs with transient edges and known ground-truth causal structure, and validate the phenomenon empirically.
\end{abstract}

\section{Introduction}
In supervised learning, a model is trained on a set of input--output pairs such that the predicted outputs closely match the ground-truth, thereby minimising a prescribed error function, typically referred to as the \textit{loss}. The trained model is then evaluated on previously unseen data, referred to as the test set. The resulting prediction error on the test set is reported as the \textit{predictive performance} of the model, on the basis of which it is benchmarked against competing approaches.
An important aspect, often overlooked by machine learning practitioners, is that
\textit{the training objective is not designed to recover the true mapping, but rather to approximate a model which best fits the given training data.}
In other words, the learned model is inherently sensitive to the sampled training data, and a substantial body of research has focused on mitigating this sensitivity through techniques such as batch training, stochastic gradient descent, batch normalisation, and regularisation, to name a few.

In this work, we revisit the problem of \textit{model sensitivity to the training data}~\cite{murphy2012machine} and instead of accepting low training loss as a proxy for successful learning, we ask the following question:
\begin{myquote}
    (Q1) What if the learned model could be compared directly against the true model?
\end{myquote}
To this end, we assume that the true model belongs to a known parametric family, and that minimising the training loss recovers the true parameters asymptotically as the number of training samples increases~\cite{vapnik2013nature}. Parameter recovery performance can then be quantified directly through the difference between the estimated and true parameters. Since, in practice we are limited by a finite number of samples, a question naturally arises:
\begin{myquote}
    (Q2) Does the observed data contain sufficient information to accurately recover the parameters?\vspace{-10pt}
\end{myquote}
A refined version of this question is:
\begin{myquote}
    (Q3) What is the relation between parameter estimation error and the amount of useful information in the observed data?
\end{myquote}
We have tools from information theory at our disposal to answer such questions, under the additional assumption that the parametric model is probabilistic.

Once the parameters have been estimated as accurately as permitted by the available data, we evaluate the predictive performance of the resulting model on unseen samples, leading to the final question:
\begin{myquote}
(Q4) How is the predictive performance of the estimated model related to the parameter estimation error?
\end{myquote}
More specifically, we study how the lowest achievable parameter estimation error is related to the predictive performance attained by the corresponding estimated model.

We now motivate the problem in the context of temporal link prediction (TLP) \cite{longa_graph_2023}, where the objective is to predict whether an edge exists at a given time based on the past history of the temporal graph. In the literature, this task is typically formulated as a binary classification problem \cite{huang2024temporal}, where the implicit assumption is that the existence or non-existence of edges in the past contains sufficient information to predict the existence of edges in the future \cite{rahman2026generating}.

However, the underlying state space of the problem grows rapidly with the size of the graph. Consider a graph with $n$ nodes and therefore $n(n-1)/2$ possible edges. Suppose the prediction of a target edge depends on the existence history of $m$ past edges. Since each past edge may either exist or not exist, the number of possible historical configurations is $2^m$. Consequently, the size of the state space associated with predicting a single edge grows exponentially in $m$.

Moreover, in practice, the relevant dependencies among edges are generally unknown a priori. As a result, the learning procedure may effectively search over all possible edges in the graph in order to identify the causal dependencies. In such settings, the number of relevant historical edges may itself scale with the size of the graph, i.e. $m = \mathcal{O}(n^2)$, yielding an effective state space of $\mathcal{O}\left(n^2 2^{n^2}\right)$, which grows exponentially with the number of edges.

\begin{myquote}
    For a graph with $5$ nodes, the effective state space exceeds $10^4$, while for a graph with $10$ nodes, it grows to over a quadrillion $(10^{15})$.
\end{myquote}

This \textit{state explosion}~\cite{valmari1996state} raises an important issue for learning and evaluation. Many asymptotic arguments in statistical learning rely on the assumption that the number of available samples grows sufficiently large relative to the size of the underlying state space. However, in temporal link prediction, the combinatorial growth of the state space makes such assumptions difficult to justify.

In the preceding paragraphs, we highlighted an assumption in TLP which has largely remained implicit due to the lack of certifiably causal temporal graph datasets. However, in a recent work \citep{rahman2026generating}, we proposed a parametric causal model for generating temporal graphs where the causal relation between an edge and its parent edges is deterministic. In this work, we instead require a probabilistic formulation, and therefore extend the deterministic model to a probabilistic causal framework. In this setting, parameter estimation corresponds to causal discovery.

\section{Motivation}
\label{sec:motivation}

We motivate our investigation through a binary logistic model, and characterise its estimation and prediction errors.

Consider a Bernoulli random variable $X$ with success probability $p(\theta)$ parametrised by $\theta \in \mathbb{R}$, such that
\begin{align}
    p(\theta) = \sigma(g(\theta)),
    \label{blm}
\end{align}
where $\sigma(\cdot)$ denotes the sigmoid function and $g(\theta)$ is the link function, differentiable with respect to $\theta$. Differentiating the log-likelihood with respect to $\theta$ (Appendix~\ref{app:fisher_scalar}), the Fisher information is
\begin{align}
    J(\theta) = p(\theta)(1-p(\theta)) g'(\theta)^2.
    \label{eq:fisher_scalar}
\end{align}
Applying the Cram\'{e}r--Rao (CR) inequality \citep[\textsection~11.10]{cover2012elements} for $N$ samples gives a lower bound on the mean squared parameter estimation error,
\begin{align}
    \mathbb{E}\left[ ( \hat{\theta} - \theta)^2\right] \geq \frac{1}{N}J(\theta)^{-1}.
    \label{eq:crb_final}
\end{align}

Now consider a set of samples from the Bernoulli model, and a prediction model which has recovered the parameter $\hat{\theta}$. Under a probabilistic prediction model, the expected binary cross-entropy (BCE) loss decomposes as
\begin{align}
    \mathcal{L}_{\text{BCE}}(\theta, \hat{\theta}) = h(p(\theta)) + D_{\text{KL}}(p(\theta) \| p(\hat{\theta})),
\end{align}
where $h(p(\theta))$ is the irreducible entropy and $D_{\text{KL}}(p(\theta) \| p(\hat{\theta})) \geq 0$ is the excess loss due to parameter mismatch, vanishing as $\hat{\theta} \to \theta$ \citep[\textsection~2.6]{cover2012elements}. The analogous decomposition for the mean squared error under soft predictions $\hat{x} = p(\hat{\theta})$ is
\begin{align}
    \mathcal{L}_{\text{MSE}}(\theta, \hat{\theta}) = p(\theta)(1-p(\theta)) + (p(\theta) - p(\hat{\theta}))^2,
\end{align}
where the first term is the irreducible variance of $X$ and the second is the squared bias due to parameter mismatch.

We now look at the relation between the Fisher information and entropy in the following proposition.
\begin{proposition}
\label{prop:g_condition}
For twice-differentiable $g(\theta)$, the Fisher information $J(\theta)$ and entropy $h(p(\theta))$ are co-monotonic iff
\begin{align}
    \mathrm{sign}\Big( (1-2p)\,g'(\theta)^2 + 2g''(\theta)\Big)= \mathrm{sign}(1-2p).
    \label{eq:comonotone_condition}
\end{align}
\end{proposition}
\begin{proof}
    See Appendix~\ref{app:g_condition}.
\end{proof}

Two conditions on $g(\theta)$ are sufficient to guarantee \cref{eq:comonotone_condition}:
\begin{enumerate}
    \item If $g(\theta)$ is affine, i.e. $g''(\theta) = 0$, and $g'(\theta) \neq 0$, or
    \item If $\mathrm{sign}(g''(\theta)) = \mathrm{sign}(1 - 2p(\theta))$.
\end{enumerate}

Consequently, under the stated conditions on $g(\theta)$, a Bernoulli model with higher entropy $h(p(\theta))$ admits higher irreducible prediction error, and by \cref{eq:comonotone_condition}, simultaneously has higher Fisher information $J(\theta)$ translating to  \textit{lower} achievable parameter estimation error \cref{eq:crb_final}. 
\begin{quote}
    \textit{Equivalently, the better the parameters of the model can be estimated from the observed data, the higher the irreducible error of any prediction made from those parameters.} 
\end{quote}
We refer to this as the \textbf{estimation--prediction tradeoff}. In the remainder of this paper, we construct a causal probabilistic temporal graph model in which the tradeoff holds by construction, and validate it empirically in the context of temporal link prediction.

\begin{quote}
    \textit{The goal of this work is to show that the uncertainty inherent in the data can conflate the performance assessment of a model.}
\end{quote}

\section{Related Works}
\label{sec:related}

\citet{vignat2003analysis} analyse the relationship between the Shannon entropy and Fisher information of data $X$ sampled from a distribution whose probability density function is $f_X(x)$. However, the Fisher information is defined as
\begin{align}
    I_X = \int \left( \frac{\partial f_X(x)}{\partial x} \right)^2 \frac{1}{f_X(x)} \, dx.
\end{align}
Further, they denote the entropy power by $N_X$ defined as
\begin{align}
    N_X = \frac{1}{2 \pi e} \, \exp\left( 2 H_X \right),
\end{align}
where $H_X = - \int f_X(x) \log{ f_X(x)} \, dx$ is the entropy of $X$.
Then, an uncertainty property is established as
\begin{align}
    N_X I_X \geq 1,
\end{align}
which only defines the region where $(N_X, I_X)$ points of a distribution can exist, without defining an explicit relationship between the two quantities, allowing one to be controlled without affecting the other within constraints imposed by the distribution under study. The product of Shannon entropy power and Fisher information is used as a correlation measure for many-electron systems in \citep{romeradehesa2004}.

In our work, the distribution is parametrised, and the partial derivative in the definition of Fisher information is taken wrt the parameters and not the samples,
\begin{align}
    J(\theta) = \int \left( \frac{\partial \log{ f(x; \theta) }}{ \partial \theta} \right)^2 \, f(x; \theta) \, dx.
\end{align}
Moreover, we show that for Bernoulli distributions with link function affine in $\theta$, a higher Shannon entropy concurrently means higher Fisher information.

\citet{nokleby2016rate} bound the sample complexity of learning a Bayes classifier using the parametric Fisher information of the distribution family together with entropy terms, including a binary logistic classification case study structurally similar to \cref{blm}. However, their entropy terms are the differential entropy of the parameter itself and of the regressor's values under a Bayesian prior $q(\theta)$, rather than the entropy of the outcome $X$ at a fixed, unknown $\theta$; correspondingly, their Bayes risk is a reducible quantity that vanishes as the sample size $n \to \infty$, whereas $h(p(\theta))$ in \cref{blm} is a fixed floor independent of $n$. 

\citet{hsu2024samplecomplexity} derive bounds on the sample size needed to achieve a target parameter estimation error for a Bernoulli logistic model.

\section{Preliminaries}
We begin by formalising our representation. A temporal graph with transient edges can be viewed as a collection of irregular binary pulse signals, one per edge, that switch between active and inactive states. More precisely,
\begin{definition}[Temporal Graph]
    A temporal graph $\gG$ with transient edges and continuous timestamps over a set of nodes $\gV$ is defined as a set of quadruples $(u,v,t,t')$, where $u,v \in \gV$, $t,t' \in \mathbb{R}$ with $t < t'$. Each quadruple indicates that edge $(u,v)$ is active on the interval $[t, t')$.
\end{definition}

\begin{definition}[Binomial Random Intersection Graph]
In a binomial random intersection graph $G(n,m,p)$, each vertex $i \in [n]$ is associated with a random subset $\gM_i \subseteq [m]$, formed by including each element independently with probability $p$. An undirected edge is placed between $i$ and $j$ if $\gM_i \cap \gM_j \neq \emptyset.$ \citep[\textsection~13.1]{frieze2023random}
\end{definition}

\begin{definition}[Poisson Point Process]
A one-dimensional homogeneous Poisson point process (PPP) with rate parameter $\lambda > 0$ is a random countable set $\Phi(\lambda) \subseteq \mathbb{R}$ such that for any interval $\gB \subseteq \mathbb{R}$, the number of points in $\gB$ follows a Poisson distribution with mean $\lambda |\gB|$, where $|\gB|$ denotes the length of the interval.
\end{definition}

Any realization of $\Phi(\lambda)$ can be written as an ordered sequence $\{t_1, t_2, \ldots\}$ where $t_1 < t_2 < \cdots$. The causal model, which we introduce next, determines how these timestamps translate into edge activations and deactivations.

We adopt the causal models introduced by \citet{pearl2009causality}, which formalise how variables causally influence one another through a directed graph and associated structural equations.

The causal model consists of a set of exogenous variables denoted by $\gU= \{U_1, U_2, \ldots\}$, and a set of endogenous variables $\gX  = \{X_1, X_2, \ldots\}$, whose values are determined within the model. The causal structure is encoded in a directed acyclic graph $\graph$, where an edge from $A$ to $B$ indicates that $A$ causally influences $B$, and $A$ is called a \textit{parent} of $B$. We denote the set of parents of $X_i$ by $\parents_i$.

\begin{definition}[Causal Model]
    A causal model is a pair $( \graph, \Theta )$ consisting of a causal graph $\graph$ and a set of parameters $\Theta$ compatible with $\graph$. The parameters $\Theta$ assign a function $x_i = f_i(\pa_i, u_i)$ to each variable $X_i \in \gX$ and a probability  $\mathbb{P}(u_i)$ to each $u_i$.
\end{definition}

In our setting, we use a probabilistic variant where structural equations specify conditional distributions rather than deterministic mappings. To this end, we write the probabilistic structural equation model (SEM) specifying the conditional distribution of $X_i$ given $\pa_i$ and $u_i$ as
\begin{align}
    \mathbb{P}(X_i = x_i \mid \parents_i = \pa_i, U_i = u_i) = f_i(x_i \mid \pa_i, u_i).
\end{align}

\section{Random Causal Temporal Graphs}
\label{sec:rctg}

We now consider the problem of generating causal temporal graphs with transient edges. Unlike ephemeral edges, which exist only instantaneously, transient edges persist for some finite duration before vanishing. Our approach extends our earlier model \citep{rahman2026generating}, originally proposed for ephemeral edges, to the transient case.

\begin{quote}
The key premise is that the presence or absence of certain edges in the \textit{recent past} may cause the activation or deactivation of other edges.
\end{quote}

Let $t$ denote the present time, then $t - \varepsilon$ for some infinitesimal $\varepsilon > 0$, denotes the recent past. To model the timestamps at which edges activate and deactivate, we use one-dimensional homogeneous Poisson point processes~\cite{haenggi2012stochastic}, one process per edge.

We define a causal model that captures the influence of past edge states on future edge activations. For a temporal graph with $n$ nodes, there are $E = \binom{n}{2}$ possible edges. The causal graph $\graph$ is defined over the following sets of binary variables:
\begin{itemize}[noitemsep, topsep=0pt]
    \item $\gU = \{U_i(t) : i \in [E]\}$, where $U_i(t)$ indicates whether a Poisson trigger occurs for edge $i$ at time $t$,
    \item $\gX' = \{X_i'(t) : i \in [E]\}$, where $X_i'(t)$ denotes the state of edge $i$ at time $t - \varepsilon$,
    \item $\gX = \{X_i(t) : i \in [E]\}$, where $X_i(t)$ denotes whether edge $i$ is active at time $t$.
\end{itemize}
At time $t$, the variables $\gX$ are treated as endogenous, while $\gU$ and $\gX'$ are treated as exogenous. At each time $t$, we treat $\gX'$ and $\gU$ as given inputs to determine $\gX$. Temporally, $\gX'$ represents the edge states from the previous instant.

We begin by constructing a binomial random intersection graph $G(E,m,p)$. Each undirected edge in $G$ is then assigned an orientation independently and uniformly at random, yielding a directed graph. Let $\rmA \in \{0,1\}^{E \times E}$ denote its adjacency matrix.

The causal graph $\graph$ is defined over the variables $\gX' \cup \gX \cup \gU$, where the current state $X_i(t)$ depends on (i) the previous states $X_j'(t)$ of its parent edges, as specified by $\rmA$, and (ii) the trigger variable $U_i(t)$.

Ordering the variables as $(\gX', \gX, \gU)$, with each block arranged by edge index from $1$ to $E$, the adjacency matrix of the causal graph takes the block form
\begin{align}
\rmA^\graph =
\begin{bmatrix}
\bm{0} & \rmA & \bm{0} \\
\bm{0} & \bm{0} & \bm{0} \\
\bm{0} & \rmI & \bm{0}
\end{bmatrix}.
\label{causal_graph}
\end{align}

For each edge $i \in [E]$, we define an independent homogeneous PPP $\Phi_i$ with rate $\lambda_i \in \mathbb{R}^+$ over $[0, T)$, where $T \in \mathbb{R}^+$. We denote the rate parameters collectively as $\bm{\lambda} = \{\lambda_1, \ldots, \lambda_E\}$. We also define the causal influence matrix $\Theta \in \mathbb{R}^{E \times E}$ where the entry $\Theta_{j,i}  =0 \iff \rmA_{j,i}=0$.

At any timestamp $t$, the state of the edge is sampled via Bernoulli trial with success probability 
\begin{align}
    p_i^t &= \mathbb{P}\left(X_i(t) = 1 \mid  u_i(t) , x_1'(t), \cdots, x_E'(t) \right) \nonumber\\
    &= u_i(t) \cdot  \sigma \left(\Theta_{:,i}^\top \, \rvx'(t) \right) + (1 - u_i(t))\cdot x_i'(t),
    \label{psem0}
\end{align}
where $\Theta_{:,i}$ denotes the $i$-th column of the parameter matrix and $\rvx'(t) = [x_1'(t), \ldots, x_E'(t)]^\top \in \{0,1\}^E$ is the vector of all edge states immediately before time $t$.

In this work, we sample the elements of $\Theta$ independently from a zero-mean normal distribution,
\begin{align}
    \Theta_{j,i} \sim \gN(0, \sigma^2), \quad \forall i,j \in [E] : \rmA_{j,i} = 1,
\end{align}
and $\Theta_{j,i} = 0$ whenever $\rmA_{j,i} = 0$.

At trigger times $t \in \Phi_i$, \cref{psem0} reduces to the Bernoulli generalised linear model (GLM)~\citep{dobson2018introduction},
\begin{align}
    \mathbb{P}\left( X_i(t) = 1 \mid \rvx'(t) \right)
    = \sigma\!\left(\Theta_{:,i}^\top \rvx'(t)\right),
\end{align}
which defines the probabilistic structural equation for events of edge $i$.

We represent the causal data-generating mechanism by the stochastic generative model
\begin{align}
    \mathscr{M}(n, m, p, \sigma^2, T, \bm{\lambda}),
\end{align}
parameterised by the tuple $(n, m, p, \sigma^2, T, \bm{\lambda})$, where $n$ is the number of nodes; $(m,p)$ parameterise the binomial random intersection graph $G(E,m,p)$ from which the causal graph $\graph$ is constructed; $\sigma^2$ is the variance of the Gaussian distribution used to sample the non-zero entries of $\Theta$; $T$ specifies the time horizon $[0,T)$;  and $\bm{\lambda}=\{\lambda_1,\ldots,\lambda_E\}$ denotes the PPP intensities for the $E$ edges.

Sampling from $\mathscr{M}$ proceeds in two stages. First, we sample the causal model structure, parameters, and triggers:
\begin{align}
    \mathscr{C}(\rmA^\graph, \Theta, \bm{\Phi}) \sim \mathscr{M}(n, m, p, \sigma^2, T, \bm{\lambda}),
\end{align}
where $\rmA^\graph$ is the causal graph adjacency matrix \cref{causal_graph}, $\Theta$ is the causal influence matrix, and $\bm{\Phi} = \{\Phi_1, \ldots, \Phi_E\}$ are the PPP realisations. Second, we generate the temporal graph by traversing $\cup_{i \in [E]} \Phi_i$ chronologically, and evaluating the probabilistic structural equations for the triggered edge. Sampling a temporal graph $\gG$ from the causal model $\mathscr{C}(\rmA^\graph, \Theta, \bm{\Phi})$ is denoted by
\begin{align}
    \gG \sim \mathscr{C}(\rmA^\graph, \Theta, \bm{\Phi}).
\end{align}
The complete generative procedure is presented in \cref{algo1}.

\paragraph{Estimation \& Prediction}
Given a realised temporal graph
\begin{align}
    \gG \sim \mathscr{C}(\rmA^\graph,\Theta,\bm{\Phi}),
\end{align}
our goal is to recover the causal influence matrix $\Theta$ from the observed edge states. Since the structural equation for each edge is a Bernoulli GLM evaluated at the trigger times $\Phi_i$, we estimate each column $\Theta_{:,i}$ independently.

The maximum-likelihood estimator of $\Theta_{:,i}$ is equivalently obtained by minimising the binary cross-entropy loss,
\begin{align}
    \hat{\Theta}_{:,i} = \argmin_{\bm{\theta} \in \mathbb{R}^E}
    - \sum_{t \in \Phi_i} \Big( x_i(t) \log \sigma\left(\bm{\theta}^\top \rvx'(t)\right)
    \nonumber \\
    + \left(1 - x_i(t)\right) \log\left(1 - \sigma\left(\bm{\theta}^\top \rvx'(t)\right)\right) \Big).
    \label{eq:estimation}
\end{align}
Thus, \cref{eq:estimation} depicts unregularised logistic regression, solved for each edge $i \in [E]$. This formulation ensures consistency, i.e., as $|\Phi_i| \to \infty$, the estimator $\hat{\Theta}_{:,i}$ converges to the true $\Theta_{:,i}$ under standard regularity conditions.

\cref{prop:crb} presents the CR bound for the temporal graph model under study.

\begin{proposition}
\label{prop:crb}
Let $\hat{\Theta}$ be any unbiased estimator of $\Theta$. Then, the parameter estimation error 
\begin{align}
\mathbb{E}[\lVert\hat{\Theta} - \Theta\rVert_F^2] \geq \sum_{i \in [E]} \mathrm{tr}\!\left(\rmJ_i^{-1}\right).
\label{crbound1}
\end{align}
where $\rmJ_i = \sum_{t \in \Phi_i} p_i^t\left(1 - p_i^t\right) \rvx'(t)\rvx'(t)^\top$ is assumed to be non-singular.
\end{proposition}
\begin{proof}
    See Appendix~\ref{app:proof_crb}.
\end{proof}

The maximum-likelihood estimator obtained by solving \cref{eq:estimation} is asymptotically unbiased but biased at finite $|\Phi_i|$; we use it as an empirical approximation to $\hat\Theta$ \cref{crbound1} in \cref{sec:results}.

\begin{remark}
For each edge $i\in[E]$, the link
\begin{align}
g_i^t(\Theta_{:,i})=\Theta_{:,i}^\top \rvx'(t)
\label{affine}
\end{align}
is affine in $\Theta_{:,i}$. 
Hence, by \cref{prop:g_condition}, the scalar factor $p_i^t(1-p_i^t)$ governing the Fisher information $J_i^t$, and the irreducible predictive entropy $h(p_i^t)$, are comonotone at every trigger time $t \in \Phi_i$.
\end{remark}

However, this per-trigger comonotonicity does not extend to the sums $\rmJ_i$ and $\sum_{t \in \Phi_i} h(p_i^t)$ aggregated over all triggers. We therefore investigate the estimation--prediction tradeoff at the aggregate level empirically in the following section.

\section{Results}
\label{sec:results}

\paragraph{Parameter settings}
Temporal graphs are generated using $\mathscr{M}$ across multiple parameter regimes. The number of nodes is fixed to $n=5$, yielding $E=10$ possible edges. The causal graph is generated as a binomial random intersection graph with $m=20$ and $p \in \{0.1,\ldots,0.9\}$. For each edge, an independent PPP is sampled over $[0, T)$ with $T=10^4$, and the causal parameters are drawn from $\mathcal{N}(0,\sigma^2)$ with $\sigma \in \{1.0, 2.0, 2.5, 3.0, 4.0, 5.0\}$.

In \cref{fig:est_fisher}, we plot the empirical parameter estimation error $\mathbb{E}[\|\hat{\Theta} - \Theta\|_F^2]$
against the CR bound $\sum_{i \in [E]}\mathrm{tr}(\mathbf{J}_i^{-1})$  across different parameter configurations of $\mathscr{M}$.
The estimation error increases with the CR bound, consistent with \cref{crbound1}.
\begin{figure}[h!]
    \centering
    \includegraphics[width=\columnwidth]{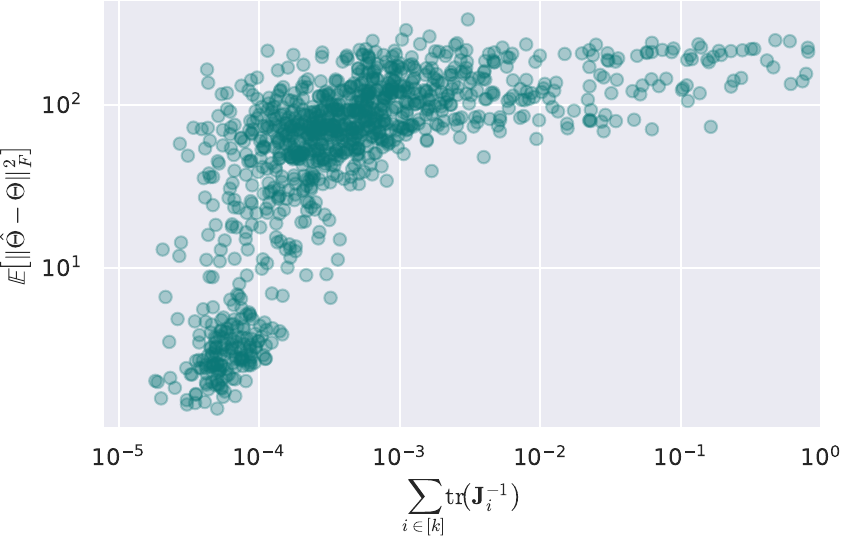}
    \caption{Empirical parameter estimation error vs. CR bound.}
    \label{fig:est_fisher}
\end{figure}

\cref{fig:ent_fisher} shows the irreducible prediction error $\sum_{i \in [E]}\sum_{t \in \Phi_i} h(p_i^t)$ against the same CR bound. The sharp inverse relationship reproduces the tradeoff. The causal model for which parameters are estimated most accurately also has the largest irreducible prediction error.
\begin{figure}[h!]
    \centering
    \includegraphics[width=\columnwidth]{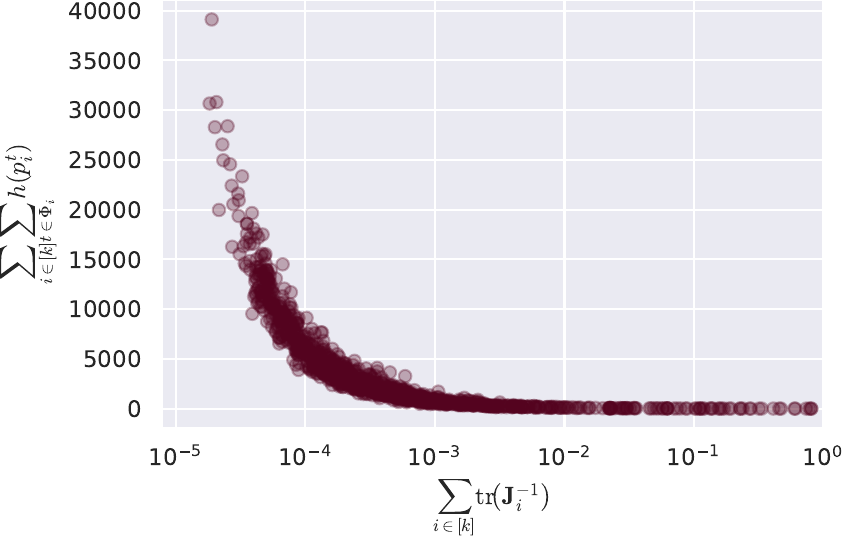}
    \caption{Irreducible prediction error vs. CR bound. }
    \label{fig:ent_fisher}
\end{figure}

Standard metrics may therefore conflate model performance with irreducible uncertainty and rank models by the entropy of the dataset rather than their ability to recover the data-generating mechanism, which is why temporal graph models with explicit generative structure matter for principled benchmarking.

\section{Conclusion}
\label{sec:conclusion}

In the probabilistic temporal graph model constructed in this paper, learning and prediction are limited by different aspects of the same data-generating process. Higher Fisher information lowers the CR bound on the parameter estimation error, but at the same time, coincides with higher entropy, which makes predictions intrinsically difficult.

We define a probabilistic causal generator with transient edges and known ground-truth parameters, enabling joint evaluation of predictive error and causal parameter recovery. The experiments confirm the relationship between estimation error and irreducible predictive error.

Therefore, predictive accuracy alone is not a reliable proxy for learning the causal mechanisms in temporal data. Temporal graph benchmarks should separate reducible model error from irreducible uncertainty and evaluate parameter recovery whenever ground truth is available. We hope this work motivates the temporal graph learning community \citep{yitgb} to integrate causality into model evaluation and look beyond standard predictive metrics.

Our construction fixes the link function to be affine in the parameters, which guarantees per-trigger comonotonicity between the Fisher information and entropy by \cref{prop:g_condition}; the aggregate estimation--prediction tradeoff shown in \cref{sec:results} is validated empirically rather than following automatically from this per-trigger property (\cref{sec:rctg}). 

The link function can also be parametrised by a neural network to define non-linear SEMs (see Appendix~\ref{app:nn}). The parameters of such a network can be deliberately initialised so that the model violates \cref{eq:comonotone_condition}, providing a constructive setting in which the estimation--prediction tradeoff may not hold.

\bibliography{causality}
\bibliographystyle{icml2023}

\appendix
\section{Fisher Information}
\label{app:fisher_scalar}

For the model in \cref{blm}, the log-likelihood of a single observation $x$ is
\begin{align}
    \ell(\theta; x) = x \log{p(\theta)} + (1-x)\log{(1 - p(\theta))}.
\end{align}
Differentiating with respect to $\theta$ and applying the chain rule through $\sigma(\cdot)$,
\begin{align}
    J(\theta) &= \mathbb{E}\left[\left(\frac{\partial \ell(\theta; x)}{\partial \theta}\right)^2\right] \nonumber \\
    &= \mathbb{E}\left[(x-p(\theta))^2 \left(\frac{\partial g(\theta)}{\partial \theta}\right)^2\right] \nonumber\\
    &= p(\theta)(1-p(\theta)) \left(\frac{\partial g(\theta)}{\partial \theta}\right)^2,
\end{align}
where the last step uses $\mathrm{Var}(X) = p(\theta)(1-p(\theta))$ and the fact that $\frac{\partial g(\theta)}{\partial \theta}$ is deterministic given $\theta$.

\section{Proof of \cref{prop:g_condition}}
\label{app:g_condition}

From \cref{eq:fisher_scalar}, $J(\theta) = \phi(p(\theta))\,g'(\theta)^2$ with $\phi(p)=p(1-p)$. Differentiating via the product and chain rules, and using $p'(\theta) = \phi(p(\theta))\,g'(\theta)$,
\begin{align}
    \frac{dJ}{d\theta} &= \phi'(p)\,p'(\theta)\,g'(\theta)^2 + \phi(p)\cdot 2g'(\theta)g''(\theta) \nonumber \\
    &= \phi'(p)\,\phi(p)\,g'(\theta)^3 + 2\phi(p)\,g'(\theta)g''(\theta) \nonumber \\
    &= \phi(p)\,g'(\theta)\big[(1-2p)\,g'(\theta)^2 + 2g''(\theta)\big],
    \label{eq:dJ_dtheta}
\end{align}
For the entropy side, $h'(p) = \log\frac{1-p}{p}$ has the same sign as $1-2p$ for all $p\in(0,1)$, so
\begin{align}
    \frac{dh}{d\theta} = h'(p(\theta))\,p'(\theta) = h'(p(\theta))\,\phi(p(\theta))\,g'(\theta),
\end{align}
and since $\phi(p)>0$, $\mathrm{sign}(dh/d\theta) = \mathrm{sign}(g'(\theta))\cdot\mathrm{sign}(1-2p(\theta))$. Comparing signs of $dJ/d\theta$ and $dh/d\theta$ (both carry an overall factor $\phi(p)g'(\theta)$ of sign $\mathrm{sign}(g'(\theta))$, since $\phi(p)>0$) reduces comonotonicity to \cref{eq:comonotone_condition}.

\newpage
\section{Proof of \cref{prop:crb}}
\label{app:proof_crb}

We proceed in two steps: first deriving the CR bound for a single edge $i \in [E]$, then summing over all edges.

\paragraph{Single edge}
Fix $i \in [E]$. At each trigger $t \in \Phi_i$, the activation $X_i(t) \in \{0,1\}$ is a Bernoulli random variable with success probability $p_i^t = \sigma(\Theta_{:,i}^\top \rvx'(t))$, conditioned on the context $\rvx'(t)$. The
log-likelihood contribution of a single observation is
\begin{align}
    \ell\!\left(\Theta_{:,i}; X_i(t)\right) = X_i(t) \log p_i^t +
    \left(1 - X_i(t)\right) \log\!\left(1 - p_i^t\right).
\end{align}
Differentiating with respect to $\Theta_{:,i}$ and applying the identity
$\sigma'(z) = \sigma(z)(1 - \sigma(z))$, the score is
\begin{align}
    \frac{\partial \ell}{\partial \Theta_{:,i}} = \left(X_i(t) - p_i^t\right)\rvx'(t).
\end{align}
The Fisher information matrix for a single observation is
\begin{align}
    \rmJ_i^t &= \mathbb{E}\left[\frac{\partial \ell}{\partial \Theta_{:,i}}
    \frac{\partial \ell}{\partial \Theta_{:,i}}^\top \right] = \mathbb{E}\left[\left(X_i(t) - p_i^t\right)^2\right] \rvx'(t)\rvx'(t)^\top \nonumber \\
    &= p_i^t\left(1 - p_i^t\right) \rvx'(t)\rvx'(t)^\top,
\end{align}
where we have used $\mathrm{Var}(X_i(t)) = p_i^t(1-p_i^t)$ and the fact that $\rvx'(t)$ is deterministic given $\Theta_{:,i}$. Since the $N_i = |\Phi_i|$ observations are independent, the total Fisher information matrix accumulates as
\begin{align}
    \rmJ_i = \sum_{t \in \Phi_i} p_i^t\left(1 - p_i^t\right) \rvx'(t)\rvx'(t)^\top.
\end{align}
Applying the multivariate CR bound~\citep[\textsection~11.10]{cover2012elements} to any unbiased estimator $\hat{\Theta}_{:,i}$,
\begin{align}
    \mathbb{E}\left[\left\lVert \hat{\Theta}_{:,i} - \Theta_{:,i} \right\rVert^2\right]
    \geq \mathrm{tr}\!\left(\rmJ_i^{-1}\right).
\end{align}

\paragraph{All edges}
The optimisation problem in \cref{eq:estimation} is solved for each $i \in [E]$, and the total squared error decomposes as
\begin{align}
    \mathbb{E}\left[\left\lVert \hat{\Theta} - \Theta \right\rVert_F^2\right]
    = \sum_{i \in [E]} \mathbb{E}\left[\left\lVert \hat{\Theta}_{:,i} -
    \Theta_{:,i} \right\rVert^2\right],
\end{align}
where we have used the fact that the Frobenius norm decomposes column-wise. Applying the per-edge CR bound to each term and summing gives the stated lower bound.

\newpage
\section{Algorithm}
\label{app:algo}

\begin{algorithm}[h!]
\caption{Sampling from $\mathcal{M}(n,m,p,\sigma^2,T, \bm{\lambda})$}
\label{alg:causal-temporal-graph}
\begin{algorithmic}[1]
\REQUIRE Number of nodes $n$, parameters $(m,p)$, variance $\sigma^2$, time horizon $T$, rates $\bm{\lambda}$
\ENSURE Temporal graph $\gG$ and ground-truth causal model $(\rmA^\graph,\Theta, \bm{\Phi})$

\STATE Set $E \gets {n \choose 2}$ and initialize $\gG \gets \emptyset$
\STATE Sample a binomial random intersection graph on $E$ edge variables with $m$ objects and probability $p$
\STATE Orient each edge independently to obtain $\rmA$

\STATE Initialize $\Theta \gets \bm{0}_{E \times E}$

\FOR{$i,j \in [E]$ such that $\rmA_{j,i}=1$}
    \STATE Sample $\Theta_{j,i} \sim \mathcal{N}(0,\sigma^2)$
\ENDFOR

\FOR{$i \in [E]$}
    \STATE Sample trigger times $\Phi_i \sim \mathrm{PPP}(\lambda_i)$ on $[0,T)$
\ENDFOR

\STATE Initialize $\rvx'(0) \gets \bm{0}$
\STATE Sort all trigger events
$\bigcup_{i \in [E]} \{(i,t): t \in \Phi_i\}$ in increasing time order

\FOR{each trigger event $(i,t)$ in chronological order}
    \STATE $\rvx'(t)$ is the edge-state vector immediately before $t$
    \STATE Set $p_i^t \gets \sigma(\Theta_{:,i}^{\top}\rvx'(t))$
    \STATE Sample $x_i(t) \sim \mathrm{Bernoulli}(p_i^t)$

    \IF{$x_i(t)=1$ and $x'_i(t)=0$}
        \STATE Record activation time $s_i \gets t$ for edge $i$
    \ELSIF{$x_i(t)=0$ and $x'_i(t)=1$}
        \STATE Add interval $(e_i,s_i,t)$ to $\gG$
    \ENDIF
\ENDFOR

\FOR{each edge $i$ with $x'_i(T)=1$}
    \STATE Add interval $(e_i,s_i,T)$ to $\gG$
\ENDFOR

\STATE \textbf{return} $\gG$ and $(\rmA^{\graph},\Theta,\Phi)$
\end{algorithmic}
\label{algo1}
\end{algorithm}

\newpage
\section{Future Work}
\label{app:nn}
A natural extension would be to replace the simple parametric structural equations with neural networks, while retaining a prescribed causal graph $\rmA$ from $\gX'$ to $\gX$. This would allow temporal graphs to be generated from more flexible nonlinear causal mechanisms.

We formalize the problem of constructing a masked neural network \cite{germain2015made} whose induced input-output connectivity exactly matches a prescribed causal adjacency matrix $\rmA$.

Let the causal graph be defined over $E$ input nodes and $E$ output nodes, with adjacency matrix $A \in \{0,1\}^{E \times E}$, where $A_{j,i} = 1$ if and only if input node $j$ is a parent of output node $i$. We seek a neural network with $L$ hidden layers, indexing the input layer as layer $0$ and the output layer as layer $L+1$. Let $n_l$ denote the number of nodes in hidden layer $l$, for $l \in [L]$, with $n_0 = n_{L+1} = E$.

The number of nodes in every hidden layer must be at least $E$ to prevent causal collapse,
\begin{align}
    n_l \geq E, \quad \forall l \in [L].
    \label{eq:width_constraint}
\end{align}

For each consecutive pair of layers $(l, l+1)$, let
\begin{align}
    \rmA^{(l, l+1)} \in \{0,1\}^{n_l \times n_{l+1}}
\end{align}
denote the bipartite adjacency matrix specifying which connections between layer $l$ and layer $l+1$ are retained.

The composition of layerwise connectivity across all hidden layers must reproduce the prescribed causal adjacency $\rmA$ exactly, under Boolean matrix composition:
\begin{align}
    \rmA^{(0,1)} \circ \rmA^{(1,2)} \circ \cdots \circ \rmA^{(L, L+1)} = \rmA,
    \label{eq:composition_constraint}
\end{align}
where $\circ$ denotes Boolean matrix multiplication, i.e. for matrices $\rmB \in \{0,1\}^{p \times q}$ and $\rmC \in \{0,1\}^{q \times r}$,
\begin{align}
    (\rmB \circ \rmC)_{j,i} = \bigvee_{k=1}^{q} \left( \rmB_{j,k} \wedge \rmC_{k,i} \right).
\end{align}
The constraint in \cref{eq:composition_constraint} encodes the requirement that a path exists from input node $j$ to output node $i$, through some sequence of unmasked connections across the $L$ hidden layers, if and only if $\rmA_{j,i} = 1$.

Given $\rmA$ and $L$, find a tuple (if exists)
\begin{align}
    \left( n_1, \ldots, n_L, \, \rmA^{(0,1)}, \rmA^{(1,2)}, \ldots, \rmA^{(L,L+1)} \right)
\end{align}
satisfying \cref{eq:width_constraint,eq:composition_constraint}, other than the trivial solution 
\begin{align}
    n_l = E; \; \rmA^{(l,l+1)} = \rmI, \quad \forall l \in [L].
\end{align}

\end{document}

%% file: math_commands.tex
\usepackage{amsmath,amsfonts,bm}
\usepackage{pifont}

\usepackage{amssymb}
\usepackage{mathtools}
\usepackage{adjustbox}
\usepackage{rotating}

\usepackage{mathrsfs}

\RequirePackage{algorithm}
\RequirePackage{algorithmic}

\def\eqref#1{equation~\ref{#1}}

\def\1{\bm{1}}

\def\rvx{{\mathbf{x}}}

\def\rmA{{\mathbf{A}}}
\def\rmB{{\mathbf{B}}}
\def\rmC{{\mathbf{C}}}

\def\rmI{{\mathbf{I}}}
\def\rmJ{{\mathbf{J}}}

\DeclareMathAlphabet{\mathsfit}{\encodingdefault}{\sfdefault}{m}{sl}
\SetMathAlphabet{\mathsfit}{bold}{\encodingdefault}{\sfdefault}{bx}{n}

\def\gB{{\mathcal{B}}}

\def\gG{{\mathcal{G}}}

\def\gM{{\mathcal{M}}}
\def\gN{{\mathcal{N}}}

\def\gU{{\mathcal{U}}}
\def\gV{{\mathcal{V}}}

\def\gX{{\mathcal{X}}}

\newcommand{\parents}{Pa} 
\newcommand{\pa}{pa}
\newcommand{\graph}{\mathscr{G}}

\DeclareMathOperator*{\argmin}{arg\,min}

\usepackage{amsthm}
\theoremstyle{theorem}

\newtheorem{proposition}{Proposition}[section]

\theoremstyle{definition}
\newtheorem{definition}{Definition}

\theoremstyle{remark}
\newtheorem*{remark}{Remark}

\newenvironment{myquote}
{\begin{quote}\itshape\vspace{-5pt}}
{\vspace{-5pt}\end{quote}}